%% file: generative_by_rl.tex
\documentclass[twoside,11pt]{article}

\usepackage[accepted]{ewrl_2022}

\usepackage{amsfonts, amsmath, amssymb}
\usepackage{xfrac}
\usepackage{graphicx, xcolor, colortbl}
\usepackage{footnote}
\usepackage{natbib}
\usepackage{dsfont}
\usepackage{bbm}
\usepackage{algorithm}
\usepackage{algpseudocode}
\usepackage{import}
\usepackage{mathtools}
\usepackage{bm}
\usepackage{xspace}
\usepackage{thmtools,thm-restate} %to restate lemma
\usepackage{accents}
\usepackage{framed}
\usepackage[inline]{enumitem}
\usepackage{booktabs}
\usepackage{color, colortbl}
\usepackage{natbib}
\usepackage{subcaption}

\definecolor{LightGray}{gray}{0.9}

\usepackage{multirow}

\usepackage{sty/notations}

\usepackage{todonotes}

\usepackage{minitoc}

\ewrlheading{}{}{}{Vargas Vieyra and M\'enard}

\ShortHeadings{Generative Models with Reinforcement Learning}{M. Vargas Vieyra and P. M\'enard}
\firstpageno{1}

\begin{document}

\title{Learning Generative Models with Goal-conditioned Reinforcement Learning}

\author{\name Mariana Vargas Vieyra \email mariana.vargas-vieyra@inria.fr \\
       \addr Inria, France
       \AND
       \name Pierre M\'enard \email pierre.menard@ens-lyon.fr \\
       \addr ENS Lyon, France}

\maketitle

\begin{abstract}%   <- trailing '%' for backward compatibility of .sty file
  We present a novel, alternative framework for learning generative models with goal-conditioned reinforcement learning. We define two agents, a \textit{goal conditioned agent} (GC-agent) and a \textit{supervised agent} (S-agent). Given a user-input initial state, the GC-agent learns to reconstruct the training set. In this context, elements in the training set are the \textit{goals}. During training, the S-agent learns to imitate the GC-agent while remaining agnostic of the goals. At inference we generate new samples with the S-agent. Following a similar route as in variational auto-encoders, we derive an upper bound on the negative log-likelihood that consists of a reconstruction term and a divergence between the GC-agent policy and the (goal-agnostic) S-agent policy. We empirically demonstrate that our method is able to generate diverse and high quality samples in the task of image synthesis.
\end{abstract}


\input{main/introduction}

\input{main/inference_as_rl}
\input{main/algorithm}
\input{main/experiments}
\input{main/conclusion}

% Acknowledgements should go at the end, before appendices and references

\acks{This work was granted access to the HPC/AI resources of IDRIS under the allocation 2022-AD011011232R2  made by GENCI. We gratefully acknowledge the support of the Centre Blaise Pascal's IT test platform at ENS de Lyon (Lyon, France) The platform operates the SIDUS solution \citep{quemener2013sidus} developed by Emmanuel Quemener. Pierre M\'enard acknowledges the support of the Chaire SeqALO (ANR-20-CHIA-0020-01).}

\vskip 0.2in
\bibliography{generative_by_rl-bib}

% Manual newpage inserted to improve layout of sample file - not
% needed in general before appendices/bibliography.

\newpage
\appendix
\input{appendix/detailed_derivation}
\newpage
\input{appendix/experiments_detail}

\end{document}

%% file: main/introduction.tex
%!TEX root = ../generative_by_rl.tex

\section{Introduction}

We consider the problem of learning a \emph{generative model}.
In recent years the study of generative models became a vast and prolific research field in the machine learning community.
Because of their ability to capture important information about the distribution of the available data they have the capacity to generate new samples, thus emulating the nature of the training set.
Generative models have shown outstanding results in applications such as speech generation \citep{vandenOord2016} and high-quality image generation \citep{brock2019, razavi2019}.
Recent examples of these models include Generative Adversarial Networks (GAN)~\citep{goodfellowgan2014,karrasprogressive2018,brock2019}, Variational Auto-Encoders (VAE)~\citep{rezende14,kingma2014,oord2017neural}, Flow-based Models~\citep{dinh2017,kingmagflow2018}, Auto-regressive Models~\citep{germain2015made,oordpixelcnn2016} and Diffusion Models (DM)~\citep{sohldickstein2015,hodenoising20,nichol21a,songME21}.

% \todoMariana{mention limitations of previous methods. Perhaps introduce diffusion models too?}

Given a sample from the training set, Diffusion Models work by adding a small amount of Gaussian noise to the sample in some pre-defined amount of steps, denoted by $H$.
This \emph{diffusion process} yields a sequence of $H$ intermediate representations, each of which have a larger amount of noise than the previous one.
Note that by following this process, the last representation we obtain is a sample from a Gaussian distribution.
The generative model is then obtained by learning how to reverse this diffusion process.
Once trained, the model is able to generate new samples by applying the reverse diffusion process on gaussian noise.
Recently, \citet{sohldickstein2015,hodenoising20,nichol21a,dhariwal2021diffusion} showed that DMs are competitive with GANs in the task of high-quality image synthesis.
% \todoPi{This seems disjointed from what we're so far discussing. How do we transition to this score-based approach?}
Interestingly, a dual view of DM~\citep{hodenoising20} is the score-based generative model~\citep{song2019}, where the generative model is the result of several steps of a Langevin dynamic~\citep{parisi1980nh} following a score function.
Such score function is learned beforehand in order to approximate the gradient of the log-likelihood.
Although DMs are simple to train, they are slow to sample from~\citep{hodenoising20}
as the model in general requires a large number of steps $H$ ($H \approx 1000$) to be able to reconstruct the inputs.
Several techniques have been proposed to overcome this issue, but they speed-up the sampling process in detriment of the output quality~\citep{nichol21a,songME21}.
This calls for approaches that naturally require less steps to generate samples.

In this paper we adopt a reverse point of view: we learn how to transform an arbitrary fixed initial state into any sample of the training set using the Goal-Conditioned Reinforcement Learning (GCRL) framework \citep{leslie1993,pong2018,andrychowicz2017,schaul15}.
In GCRL the agent aims at a particular state called the \emph{goal}.
At each step, the environment yields a loss that accounts for "how far" the agent is from the goal it is targeting.
For example, the loss can be defined as the Euclidean distance between the current state and the goal.
In  the context of learning a generative model we consider the training set to be a \emph{set of goals}.
Intuitively, our learning procedure consists of training a family of \emph{goal-conditioned agents} (GC-policy) that learn to reach the different elements in the training set by producing a \emph{trajectory} of intermediate representations, departing from the fixed initial state.
At the same time, we obtain the generative model by learning a \emph{mixture} policy of these goal-conditioned policies where the \emph{goal is sampled uniformly at random from the training set}.
Concretely, the generative model approximates a policy that randomly picks a trajectory that departs from the fixed initial state and leads to an element of the training set.
This can be cast as a supervised learning procedure.
Note that the goal-conditioned agents are used for training only.
At inference time, we generate trajectories with the mixture policy and collect the states reached at the final step.
Such states are the new generated samples.
Similar to variational inference~\citep{blei2017}, we derive a lower bound on the log-likelihood that consists of two terms.
The first term measures how good the GC-policy is able to reconstruct the samples.
The second term is a divergence term between the mixture policy and the GC-policy.
We also provide empirical evidence that our method is able to effectively reconstruct the observed data and generate a rich variety of new samples in a small number of steps ($H\approx 16$).

To summarize we highlight our main contributions:
\begin{itemize}[noitemsep]
  \item We bridge the gap between the fields of RL and generative models by introducing a novel framework that learns a generative model using GCRL. Although GCRL has been studied in similar contexts \citep{rudner2021outcome,attias03a}, it is the first time to the best of our knowledge that RL is used as a building block for learning generative models.
  \item We derive a lower bound on the log-likelihood that accounts for the reconstruction quality of the model.
  \item We provide an empirical evaluation that compares our method to the Variational Auto-Encoder model. Comparisons to other generative models are left as future work.
\end{itemize}

%% file: main/inference_as_rl.tex
\section{Inference as Goal-Conditioned Reinforcement Learning}
\label{sec:inference_as_GCRL}

We assume that we have access to a training set $\{\tx^1,\ldots,\tx^N\}=:\cD\subseteq\cX$ of samples from some unknown distribution. Our goal is to perform inference on this training set where the candidates probability distributions are the final state distribution of a policy in a certain episodic MDP.

\paragraph{Markov Decision Problem (MDP)} We consider a loss-free episodic MDP $\cM= (\cX,\cY,H,p)$ where $\cX$ is the set of states (that also contains the training set $\cD$), $\cY$ the action set, $H$ the number of steps, $p_h(x'|x,y)$ is the transition probability from state~$x$ to state~$x'$ by taking the action $y$ at step $h\in[H]$, where $[H] := \{1,\dots,H\}$.

\paragraph{Policy, reach probability and value function}  A policy $\pi$ is a collection of functions $\pi_h : \cX \to \Delta(\cY)$ for all $h\in [H]$, where every $\pi_h$  maps each state to a probability distribution over actions.

Under policy $\pi$ a trajectory is generated as follows: an initial state $x_1 \sim p_0$ is sampled. Then for $h\in[H]$, given the current sate $x_h$, the agent samples an action $y_h\sim\pi_h(\cdot|x_h)$ and the next state $x_{h+1}\sim p_h(\cdot|x_h,y_h)$ is generated according to the transition probability. We denote by $p_h^\pi(x)$ the probability to reach the state $x$ in the MDP $\cM$ at step $h$ under the policy $\pi$. Similarly, we denote by $p^\pi(\tau)$ the probability distribution of a trajectory $\tau = x_1,a_1,\ldots,x_H,a_H, x_{H+1}$
under the policy $\pi$. Given two policies $\pi$ and $\pi'$ the Kullback-Leibler divergence between $p^\pi$ and $p^{\pi'}$ is, by the chain rule,
\[
\KL(p^\pi, p^{\pi'}) = \E^\pi \left[ \sum_{h=1}^H \KL\!\big(\pi(x_h),\pi'(x_h)\big) \right]\,,
\]
where $\E_\pi$ is the expectation under $p^\pi$.

% , receives a loss $\ell_h = \ell_h(x_h,y_y)$
% The value functions of policy $\pi$ is
% \[
% V^\pi = \E_\pi\left[ \sum_{h=1}^H  \ell_h\right]\,,
% \]
% where $\E_\pi$ is the expectation over trajectory generated with the policy $\pi$ in the MDP $\cM$.

\paragraph{Upper bound the negative log-likelihood} We assume as probabilistic model for the training set the reach probability of the last step $p_{H+1}^\pi(\cdot)$ \emph{parameterized by the policy} $\pi$. We then want to find the policy that minimizes the negative log-likelihood of the training set $L(\pi)$ where
\[
L(\pi) := \frac{1}{N}\sum_{\tx\in\cD} \log \frac{1}{p_{H+1}^\pi(\tx)}\,.
\]
Solving this optimization problem is difficult because the probability $p_{H+1}^\pi(\tx)$ is typically intractable. Similarly to variational inference~\citep{blei2017}, we will instead minimize \emph{an upper bound} on the negative log-likelihood.
For a fixed element $\tx\in\cD$, conditioned on the state-action pair $(x_H,y_H)$ it holds that
\begin{align}
    \log \frac{1}{p_{H+1}^\pi(\tx)} &= -\log\Big( \E_\pi \big[p_{H}(\tx|x_{H}, y_{H}) \big] \Big)\nonumber\\
     &\leq \!\min_{\pi'} \E_{\pi'}\! \left[\log\frac{1}{p_{H}(\tx|x_{H}, y_{H})}\right]\! +\KL(p^{\pi'},p^{\pi})  \,. \label{eq:ub_negloglikelihood}
\end{align}
Equation~\eqref{eq:ub_negloglikelihood} follows from the variational
formula for the moment generating function (see Lemma 1 in Appendix B) and it defines an upper-bound on the negative log-likelihood of $\tx$.

We can then define a surrogate loss parameterized by a policy $\pi$ and a family of goal-conditioned policies $(\pi^{\tx})_{\tx\in\cD}$, indexed by an element of the training set as:
\begin{align}
  \Lub(\pi, (\pi^{\tx})_{\tx\in\cD} ) &:= \frac{1}{N}\sum_{\tx\in\cD}  \E_{\pi^{\tx}} \left[\log \frac{1}{p_{H}(\tx|x_{H}, y_{H})}\right]+ \KL(p^{\pi^{\tx}},p^{\pi})\label{eq:def_surrogate}\,.
\end{align}
Using the fact that $\Lub$ is an upper-bound on the loss $L$ our problem becomes:
\begin{align*}
\min_{\pi}L(\pi)&\leq \min_{\pi}\min_{(\pi^{\tx})_{\tx\in\cD}} \Lub(\pi,(\pi^x)_{x\in\cX}) \\
&= \min_{\pi} \frac{1}{N}\sum_{\tx\in\cD}  \min_{\pi^{\tx}} \E_{\pi^{\tx}} \left[\log \frac{1}{p_{H}(\tx|x_{H}, y_{H})}\right]+ \KL(p^{\pi^{\tx}},p^{\pi})\,.
\end{align*}
% Our goal in this paper is to minimize this surrogate loss $\Lub$.
At a high level the minimization of $\Lub$ goes as follows. For each point $\tx\in\cD$ in the training set with goal-conditioned RL we learn a policy $\pi^{\tx}$ not too far from $\pi$ that leads to $\tx$ with high probability. Simultaneously we train the policy $\pi$ to reproduce trajectories from the policies $\pi^{\tx}$.

\paragraph{Inner minimization: goal-conditioned RL} If we fix the policy $\pi$ and a state $\tx\in\cD$ in the training set then solving~\eqref{eq:ub_negloglikelihood} is equivalent to solving a regularized \emph{goal-conditioned RL problem}. That is, we want to find a policy $\pi^{\tx}$ close to $\pi$ that lead to the goal $\tx$ with high probability.
Let
\begin{equation}
  \label{eq:def_loss}
  \ell_h^{\tx}(x,y) := \begin{cases}
  \log\dfrac{1}{p_{H}(\tx|x, y)} &\text{ if }h=H\\
  0 &\text{otherwise}
\end{cases}\,,
\end{equation}
be a loss function parameterized by the goal $\tx$, and let $\cM^{\tx} = (\cX,\cY,H,p,\ell^{\tx})$ be an MDP.
Then the minimization of~\eqref{eq:ub_negloglikelihood} is equivalent to solving the MDP $\cM^{\tx}$ regularized by the policy $\pi$. Precisely, we have
\begin{align*}
\min_{\pi^{\tx}}\E_{\pi^{\tx}} \left[\log\frac{1}{p_{H}(\tx|x_{H}, y_{H})}\right]\! +\!\KL(\pi^{\tx},p^{\pi})
\!= \min_{\pi^{\tx}} V^{\pi^{\tx},\tx}+ \KL(\pi^{\tx},p^{\pi})\,,
\end{align*}
where $V^{\pi^{\tx},\tx}$ is the value of the goal-conditioned policy $\pi^{\tx}$ in the MDP $\cM^{\tx}$. Although this goal-conditioned RL problem has already been studied by \citet{rudner2021outcome} and \citet{attias03a}, it is the first time to the best of our knowledge that this formulation is used as an intermediate task to build a generative model.
The link with variational inference is clear: we seek a policy $\pi'$ that approximates well the posterior distribution of a trajectory generated by $\pi$ conditioned on the fact that this trajectory reaches the goal $\tx$.

\paragraph{Outer minimization: supervised learning} Now, if we fix the family of goal-conditioned policies, minimizing the surrogate loss $\Lub(\cdot,(\pi^{\tx})_{\tx\in\cD})$ over the policy $\pi$ amounts to minimizing a convex combination of Kullback-Leibler divergences
\begin{equation}
  \label{eq:supervised_loss}
\argmin_\pi \Lub(\pi, (\pi^{\tx})_{\tx\in\cD}) = \argmin_\pi  \frac{1}{N}\sum_{\tx\in\cD} \KL(p^{\pi^{\tx}},p^{\pi})
\end{equation}
This optimization problem can be efficiently solved with supervised learning.
First, we sample goals according the the empirical distribution of the training set.
Then, for each goal, we generate a trajectory with the goal-conditioned policy, and collect state-action pairs.
We then use these state-action pairs to supervise the policy $\pi$.

\paragraph{Variational agent} In the sequel we call the pair $(\pi,(\pi^{\tx})_{\tx\in\cD})$ a variational agent (V-agent) made of
$\pi$ a supervised agent (S-agent) and $(\pi^{\tx})_{\tx\in\cD}$ a goal-conditioned agent (GC-agent).

%% file: main/algorithm.tex
%!TEX root = ../generative_by_rl.tex

\section{Algorithm}
\label{sec:algorithm}
In this section we describe a particular instanciaton of MDP and V-agent and introduce an algorithm to learn the V-agent.
In what follows we assume the space of states and actions are both the real $d$-space $\cX=\R^d$, $\cY = \R^d$, for some dimension $d$.

\subsection{Instanciation}

We first describe a particular MDP and family of variational agents.

\begin{figure*}[ht]
	\centering
	\includegraphics[width=0.8\textwidth]{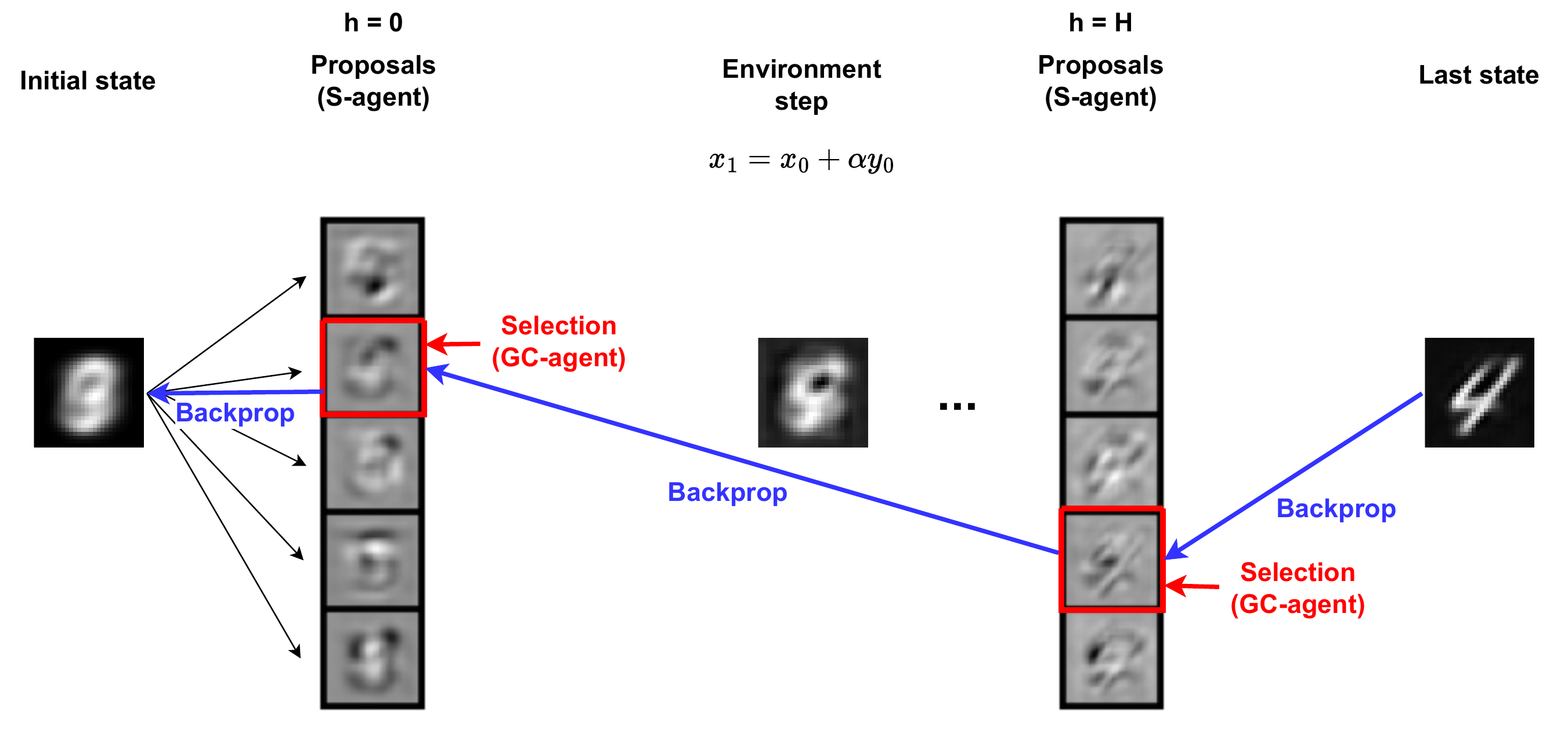}
	%\vspace{-0.5cm}
	\caption{Generation of a trajectory with the GC-agent for a sample of MNIST: At a step $h$, given a state the GC-agent select one of the proposals provided by the S-agent. The action, i.e. the selected proposal, is scaled by some rate $\alpha_h$ and added to the previous state. For a fixed sequence of actions the proposal network is differentiable. Therefore we can backpropagate the loss from the last state to the initial state once the sequence of actions is determined.}
	\label{fig:schema}
\end{figure*}

\paragraph{A particular MDP} We present a specific MPD $\cM = \{\cX,\cY,H,p\}$ parameterized by a sequence of rates $(\alpha_h)_{h\in[H]}$ and a variance $\sigma^2$. The initial state is fixed, e.g. $x_1$ is the mean of the elements in $\cD$. The first $H-1$ transition are deterministic. Precisely for $h\in[H-1]$, given the state $x_h$ and action $y_h$ the next state $x_{h+1}$ is
\begin{equation}
\label{eq:def_MDP}
x_{h+1} = x_h + \alpha_h y_h\,.
\end{equation}
The last transition probability distribution is a Gaussian distribution of variance $\sigma^2$, that is $p_H(\cdot|x_H,y_H) = \cN(x_H+\alpha_H y_H,\sigma^2)$. Note that in this case the goal-conditioned loss is proportional to the mean square error between the last state and the goal $\ell_h^{\tx}(x,y) = \ind\{h=H\}\norm{x+\alpha_h y-\tx}_2^2/(2\sigma^2)$.

\paragraph{A family of variational agents} As mentioned earlier a V-agent is made of a GC-agent and a S-agent. We only consider agents with policy given by a \emph{mixture of $A$ Dirac probability distributions} for some fixed $A\in\N$. That is, given a step and a state (and a goal for the GC-agent) the agent samples an index $a_h \in[A]$ that we call \emph{selection}. The agent also samples a vector of actions $(z_{h,a})_{a\in[A]}\in\cY^A$, that we call \emph{proposals}, from the collection of Dirac distributions. Then, the action is given by the proposal that corresponds to the selected index $y_h=z_{h,a_h}$. Note that, since the Kullback-Leibler divergence between two Dirac distribution with different supports diverges, in order to maintain the loss $\Lub$ defined in~\eqref{eq:def_surrogate} finite \emph{the GC-agent and the S-agent should share the same proposals}. In particular, the proposal function of the GC-agent that maps the step, state and goal to the collection of proposal \emph{is independent of the goal} since it is the case for the proposal function of the GC-agent.
We use GCRL to learn to \emph{select} the best proposal among the $A$ proposals.
In other words, we represent the action to be taken by its index $a_h\in [A]$.
This allows us to operate on a discrete space of actions.
Intuitively, the agent learns to map samples from the underlying distribution of the data to sequences of indexes representing the selected actions at each step.
The proposal function is fit to provide proposals that approach the agent to the goals, see below.
Therefore, the MDP is non-stationary as the proposals change during the learning.
By sharing the proposal function between both agents and representing actions by their indexes, we traded a GCRL task with continuous action space for a (non-stationary) GCRL task with discrete action space.

\begin{remark}
	We use this particular class of policies for the CG and S-agent: a mixture of Dirac distributions that allows us
	to represent a highly multi-modal distribution in few steps (provided that the rates $\alpha_h$ are large enough).
	Note that if one were to use the class of Gaussian probability distributions instead to model the policy, one would
	need much more steps to model a multi-modal distribution.
	Crucially, it would be harder to "branch out" to several different
	specific states from the current state by only adding Gaussian noise, as the model would be prone to mode collapse.
\end{remark}

\paragraph{Parametrization} We parameterize the S-agent by two neural networks: The proposal network $\Spropnet_\theta$ with weights $\theta$  that takes as input a state, step and selection $(x,h,a)$ and outputs an action $y= \Spropnet_\theta(x,h,a)$. The selection network $\Sselectnet_\phi$ with weights $\phi$ that takes as input a state, step $(x,h)$ and outputs a distributions over selections. It is important to note that the S-agent selection network and the proposal network are not conditioned on the goal. As stated above the GC-agent only needs to pick indexes $a_h\in[A]$. It is parameterized by a GC-selection Q-values network $\GCselectQnet_\psi$ with weights $\psi$ that takes as input a state, step and goal $(x,h,\tx)$ and outputs Q-values $\big(\GCselectQnet_{\psi,a}(x,h,\tx)\big)_{a\in[A]}$.
Typically for the $\Spropnet_\theta$ we choose a \UNET architecture~\citep{Ronneberger2015UNetCN} whereas the $\Sselectnet_\phi$ and $\GCselectQnet_\psi$ are simple convolutional neural networks. See Appendix~\ref{app:net_arch} for more details.

\paragraph{Unfolding the proposal network} Note that once a sequence of selections $(a_1,\ldots,a_H)$ is fixed, the last state (or more precisely, the state's mean) of a trajectory generated with actions $y_h = \Spropnet_\theta(x_h,h,a_h)$ from the proposal network and the aforementioned selections is a differential function of the weights $\theta$.
To see this one just needs to unfold \eqref{eq:def_MDP} trough the steps $h\in[H]$, see Figure~\ref{fig:schema}. We denote by $\Spropnet_{\theta,U}(a_1,\ldots,a_H)= x_H+\alpha_H y_H$ the unfolded network that maps a sequence of selections to the mean of the state at step $H+1$ when the actions are obtained with the S-agent proposal network.

\subsection{Training}\label{sec:training}The learning procedure of the V-agent is split into two parts: trajectories are sampled with the GC-agent (and the proposal from the S-agent) to reconstruct a goal sampled at random from the training set. Then these trajectories are used to learn the different networks that parameterized the GC-agent and the S-agent. In particular the GC-agent is learned with the \DQN algorithm\footnote{Precisely, in the experiments we use the \DDQN algorithm \citep{hasselt2016deep} with Dueling Q-Network \citep{wang2016dueling}.} \citep{mnih2013atari}.

\paragraph{Memory} We use three replay buffers, one for each network. The replay buffer $\replayGCselect$ associated with the GC Q-values network $\GCselectQnet$, that will be fed with transitions. The replay buffer $\replaySselect$ associated with the S-agent selection network $\Sselectnet$. And the replay buffer $\replaySprop$ of the S-agent proposal network $\Spropnet$.

\paragraph{Sample trajectory} We first sample a goal $\tx$ uniformly at random from the training set $\cD$. Then a trajectory is generated with the GC-agent as follows. At step $h$, given the current state $x_h$, we pick a selection index $a_h\in\argmin \GCselectQnet_{\psi,a}(x_h,h,\tx)$. The action $y_h = \Spropnet_{\theta}(x_h,h,a_h)$ is the output of S-agent proposal network that corresponds to the index $a_h$. The next state $x_{h+1}$ is given by \eqref{eq:def_MDP} and the GC-agent receives a loss $\ell_h=\ell_h^{\tx}(x_h,y_h)$. The observed transition is stored in the replay-buffer $\replaySprop$ of the GC-agent,
as well as the selection $(x_h,h,a_h)$ in the selection replay-buffer $\replaySselect$ of the S-agent.
Once the horizon is reached we record the selections and the goal $(\{a_1,\ldots,a_h\},\tx)$ in the proposal replay-buffer $\replaySprop$ of the S-agent. See Algorithm~\ref{alg:OurSampleTraj} for a detailed description.

\begin{algorithm}[ht]
\centering
\caption{\OurSampleTraj}
\label{alg:OurSampleTraj}
\begin{algorithmic}[1]
   \State Sample uniformly at random goal $\tx$ from the train set $\cD$.
	 \State Get initial state $x_1$.
	\For{$h \in[H]$} \Comment Generate trajectory with the GC-agent.
	\State Get selection 	$a_h \in \argmin_{a\in[A]} \GCselectQnet_{\psi,a}(x_h,h,\tx)$\Comment We can add an exploration malus.
	\State Get action $y_h = \Spropnet_\theta(x_h,h,a_h)$.
	\State Observe next state $x_{h+1} = x_h +\alpha_h y_h$.
	\State Get loss $\ell_h= \ell_h^{\tx}(x_h,y_h)$.
	\State Record transition $(x_h,h,a_h,x_{h+1},\ell_h,\tx)$ in $\replayGCselect$ and selection $(x_h,h,a_h)$ in $\replaySselect$.
	\EndFor
	\State Record trajectory selections $(\{a_1,\ldots,a_H\},\tx)$ in $\replaySprop$.
\end{algorithmic}
\end{algorithm}

\paragraph{Update} The GC-agent is learned with the \DQN algorithm \citep{mnih2013atari}. Thus we need to define a target Q-values network parameterized by the weights $\psi^{\texttt{target}}$. At a high level we use the target Q-values network $\GCselectQnet_{\psi^{\texttt{target}}}$ and the transitions stored in the replay buffer $\replayGCselect$ to compute new targets via the optimal Bellman equations. Then the GC-agent Q-values network $\GCselectQnet_\psi$ is trained to fit these targets by gradient descent on the mean squared error. Note that the targets weights are updated as an exponential moving average of the weights $\psi^{\texttt{target}}\gets \rho \psi+(1-\rho)\psi^{\texttt{target}}$ for some parameter $\rho \in(0,1)$.
The update of the S-agent is simpler. The S-agent selection network is trained to reproduce the selections picked by the GC-agent by minimizing the cross-entropy loss between the selections from the replay-buffer \replaySselect and its prediction. The S-agent unfolded proposal network $\Spropnet_{\theta,U}$ is trained to map the sequence of selections stored in the replay-buffer $\replaySprop$ to the associated goals. Thus the weights $\theta$ are updated by gradient descent on the mean square error between the output of $\Spropnet_{\theta,U}$ evaluated on the selections and the goal, see Figure~\ref{fig:schema}. A complete description is provided in Algorithm~\ref{alg:OurUpdate}.

\paragraph{Exploration} For the exploration we do not rely on the by default $\epsilon$-greedy mechanism.
As a matter of fact, it is not clear how this exploration technique would interact with the learning of the proposal network which is based on the trajectories generated by the GC-agent. Instead we propose to add a penalty to the Q-values.
This penalty is defined as the logarithm of the probability of the S-agent policy.
That is, for some parameter $\kappa>0$,
\begin{align*}
a_h\in&\argmin_{a\in[A]} \GCselectQnet_{\psi,a}(x_h,h,\tx) + \kappa \log\big(\Sselectnet_{\phi,a}(x_h,h)\big)\,.
\end{align*}
This will encourage the agent to take the selections that have not been used.
Note that this penalty does not interfere with the policy of the GC agent systematically picking the same indexes for a specific goal as long as the other indexes are used for other goals.

\begin{algorithm}[ht]
\centering
\caption{\OurUpdate}
\label{alg:OurUpdate}
\begin{algorithmic}[1]
\If{time to update $\Spropnet_{\theta}$}
\State Sample a batch $\batchSprop$ in $\replaySprop$
\State Update $\Spropnet_{\theta}$ with one step of gradient,
\begin{scriptsize}
	\[
	\nabla_\theta\frac{1}{|\batchSprop|} \sum_{(\{a_1,\ldots, a_H\},\tx)\in\batchSprop} \norm{\Spropnet_{\theta,U}(a_1,\ldots a_H)-\tx}_2^2\,.
	\]
\end{scriptsize}
\EndIf
\If{time to update $\Sselectnet_{\phi}$}
\State Sample a batch $\batchSselect$ in $\replaySselect$
\State Update $\Sselectnet_{\phi}$ with one step of gradient,
\begin{scriptsize}
\[
\nabla_\phi \frac{1}{|\batchSselect|} \sum_{(x,h,a)\in\batchSselect} -\log\big(\Sselectnet_{\phi,a}(x,h)\big)
\,.\]
\end{scriptsize}
\EndIf
\If{time to update $\GCselectQnet_{\psi}$}
\State Sample a batch $\batchGCselect$ in $\replayGCselect$
\For{$(x,h,a,x',\ell,\tx) \in \batchGCselect$}
\State Compute target
\begin{scriptsize}
\[q^{\texttt{target}}(h,x',\ell,\tx) = \ell + \min_b \GCselectQnet_{\psi^{\texttt{target}},b}(x',h+1,\tx)\]
\end{scriptsize}
\EndFor
\State Update $\GCselectQnet_{\psi}$ with one step of gradient,
\begin{scriptsize}
\begin{align*}
\nabla_\psi\frac{1}{|\batchGCselect|} \sum_{(x,h,a,x',\ell,\tx) \in \batchGCselect} \big\rVert&\GCselectQnet_{\psi,a}(x,h,\tx)\\
&\quad-q^{\texttt{target}}(h,x',\ell,\tx) \big\lVert_2^2
\,.\end{align*}
\end{scriptsize}
\If{time to update $\GCselectQnet_{\psi^{\texttt{target}}}$}
\State Update target Q-values weights $ \psi^{\texttt{target}} \gets \rho \psi +(1-\rho)\psi^{\texttt{target}}$.
\EndIf
\EndIf
\end{algorithmic}
\end{algorithm}

\subsection{Inference}
\label{subsec:inference}

To obtain as sample from the learned model one just needs to generate a trajectory using the S-agent.
Note that at inference time we do not use the GC-agent, and that the selections are entirely made by the
S-agent selection network $\Sselectnet_\phi$.
It is important to point out that once the S-agent proposal network is fixed the number of possible samples generated by the model is finite and upper-bounded by $A^H$.
For a reasonably small choice of $A$ and $H$ this upper-bound is very large and not restrictive.
Observe that this model thus produces a discrete representation of the input.
Discrete latent representations have been explored by \citet{van2017neural}.
The only source of randomness in the model comes from sampling a selection from the S-agent selection network $\Sselectnet_\phi$.
A detailed sampling procedure is provided in Algorithm~\ref{alg:OurSample}

\begin{algorithm}[ht]
\centering
\caption{\OurSample}
\label{alg:OurSample}
\begin{algorithmic}[1]
	\State Get initial state $x_1$.
	\For{$h \in[H]$} \Comment Generate trajectory with the S-agent.
	\State Get selection $a_h \sim  \Sselectnet_{\psi}(x_h,h)$.
	\State Get action $y_h = \Spropnet_\theta(x_h,h,a_h)$.
	\State Observe next state $x_{h+1} = x_h +\alpha_h y_h$.
	\EndFor
	\State Return the last state $x_{H+1}$.
\end{algorithmic}
\end{algorithm}

%% file: main/experiments.tex
%!TEX root = ../generative_by_rl.tex

\section{Experiments}
\label{sec:experiments}

We empirically assess the quality of our method on the task of image reconstruction and generation. Note that our goal here is to validate experimentally our new approach rather than providing a new state of the art method. 
To that end we conducted experiments on two publicly available datasets, namely
Fashion-MNIST~\citep{fashionmnist} and MNIST~\citep{deng2012mnist}.
In what follows we describe the experimental setup and present the results.

\subsection{Setup}
We compare our method with two types of generative models, Variational Auto-Encoders and Diffusion Models.
For the VAE baselines we use a convolutional VAE architecture with the state-of-the-art hyperparameters as \citet{Subramanian2020}.
For the DM baselines we use the publicly available implementation by \citet{song2021scorebased}\footnote{\url{https://github.com/yang-song/score_sde}}
to run experiments. In particular, we fix the number of time-steps to $1000$ and the forward diffusion process hyperparameters as in~\citet{hodenoising20}.
For the sake of comparison we parameterize the reverse process with the same architecture as in our method, that is, a four-layers \UNET~\citep{Ronneberger2015UNetCN} where each layer consists of convolution layers followed by group normalization and a ReLU activation function.
To model time-steps we do as follows: At each time-step $t\in[0, \dots, 1000]$ we use an embedding layer to compute an embedding the size of the image. We then reshape and concatenate this embedding as a set of extra channels to the input image.

For training our method we generated goal-conditioned trajectories as described in Section~\ref{sec:algorithm}.
We fixed the horizon $H=16$ and the number of proposed actions per step $A = 16$.
The sequence or rates $(\alpha_h)_{h\in[H]}$ is defined such that $\alpha_h = 1/h$ for $h=1,\dots,H$.
Similar to DMs, we parameterize the agents' proposal network with a four-layers \UNET,
where the dependence on the current step $h$ is also modeled with an embedding layer.
Both the GC-agent and S-agent selection networks are parameterized as Convolutional Neural Networks with three blocks each.
Further details on the architecture of our model and the baselines can be found in Appendix~\ref{app:net_arch}.
The rest of hyperparameters are chosen as follows:
The exploration parameter is defined as $\kappa = 0.05$, and $\rho = 0.06$.
We fixed the initial state to zero. The models were optimized using Adam optimizer~\citep{Kingma2015AdamAM} with a learning rate fixed to $0.0001$.
All models were trained until convergence.

\subsection{Results}

To compare our method with the baselines we collect the following quantities: Mean Squared Error (MSE) loss between the test
set and the test set reconstruction (Rec.), Fr\'echet Inception Distance (FID)~\citep{heusel2017gans} between the test set and its reconstruction (Rec.),
and FID of randomly generated samples with respect to the training set (Samples).
For the VAE, a random sample can be generated by feeding a sample from an isotropic Gaussian distribution (the prior of the VAE latent space) to the decoder.
To sample from the trained DMs we used the sampling method described by \citet{song2021scorebased}.
For our method, random samples are generated as described in Section~\ref{subsec:inference}.

\begin{table}[h!]
  \centering
  \begin{tabular}{lccccc}
    \toprule
    \multirow{2}{*}{Dataset} &
    \multirow{2}{*}{Method} &
      MSE &
      \multicolumn{2}{c}{FID} \\
      \cmidrule(r){4-5}
    &  & Rec. & Rec. & Samples \\
      \midrule
  \multirow{3}{*}{Fashion-MNIST} &  VAE & .13 & 45.09 & 60.06 \\
    & DM & \--- & \--- & \textbf{9.63} \\
    & Us & \textbf{.10} & \textbf{39.10} & 47.74 \\
    % D3 & 5.5 &  5.5 & 5.5 \\
    \midrule
    \multirow{3}{*}{MNIST} &  VAE & .12 & 16.49 & 27.96 \\
      & DM & \--- & \--- &  \textbf{1.94}\\
      & Us & \textbf{.06} &  \textbf{7.77} & 19.46 \\
    \bottomrule
  \end{tabular}
  \caption{Evaluation of VAE, DM, and our model by MSE between the test set and test sample reconstruction, FID (lower is better) 
  between test set and test sample reconstruction, and FID of randomly generated samples with respect to the training set.}
  \label{tab:results}
\end{table}

\begin{figure}[h!]
  \centering
  \begin{subfigure}[b]{0.45\textwidth}
     \includegraphics[width=1\linewidth]{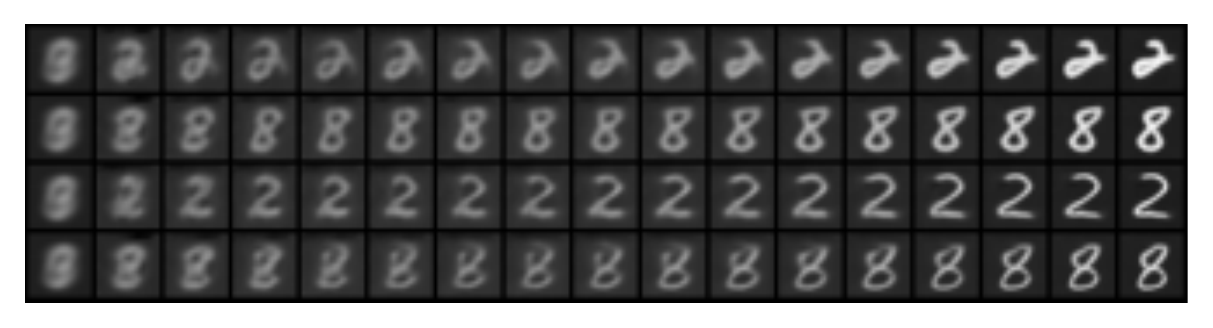}
     \caption{MNIST}
     \label{fig:Ng1}
  \end{subfigure}
  \begin{subfigure}[b]{0.45\textwidth}
     \includegraphics[width=1\linewidth]{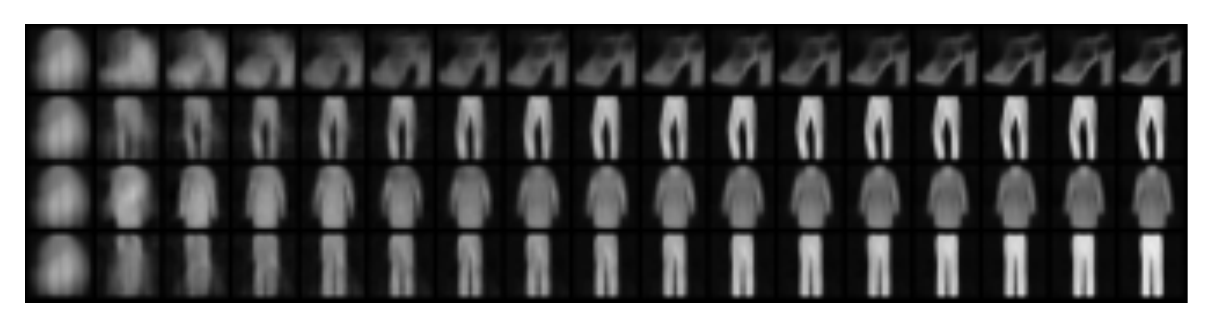}
     \caption{Fashion-MNIST}
     \label{fig:Ng2}
  \end{subfigure}
  \caption[]{Four trajectories generated by the S-agent for MNIST (left) and Fashion-MNIST (right). Note that all trajectories depart from the same initial state. The source or randomness comes from the selection function of the agents.}
  \label{fig:traj}
\end{figure}

\begin{figure*}[h!]
  \centering
  \begin{subfigure}[b]{0.3\textwidth}
     \includegraphics[width=1\linewidth]{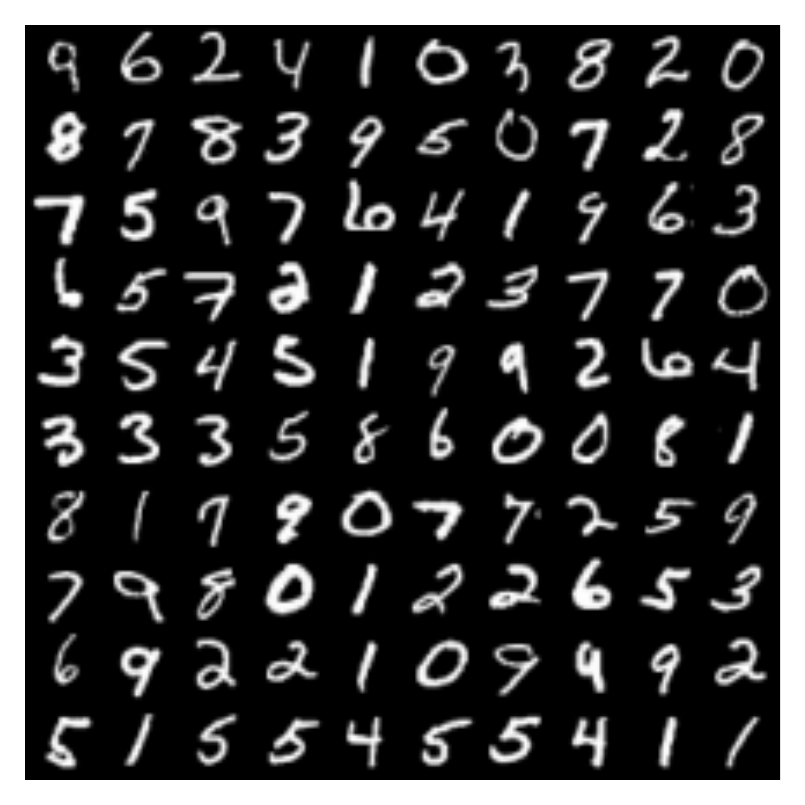}
     \caption{Batch of goals from MNIST test set.}
     \label{fig:Ng1}
  \end{subfigure}
  \begin{subfigure}[b]{0.3\textwidth}
     \includegraphics[width=1\linewidth]{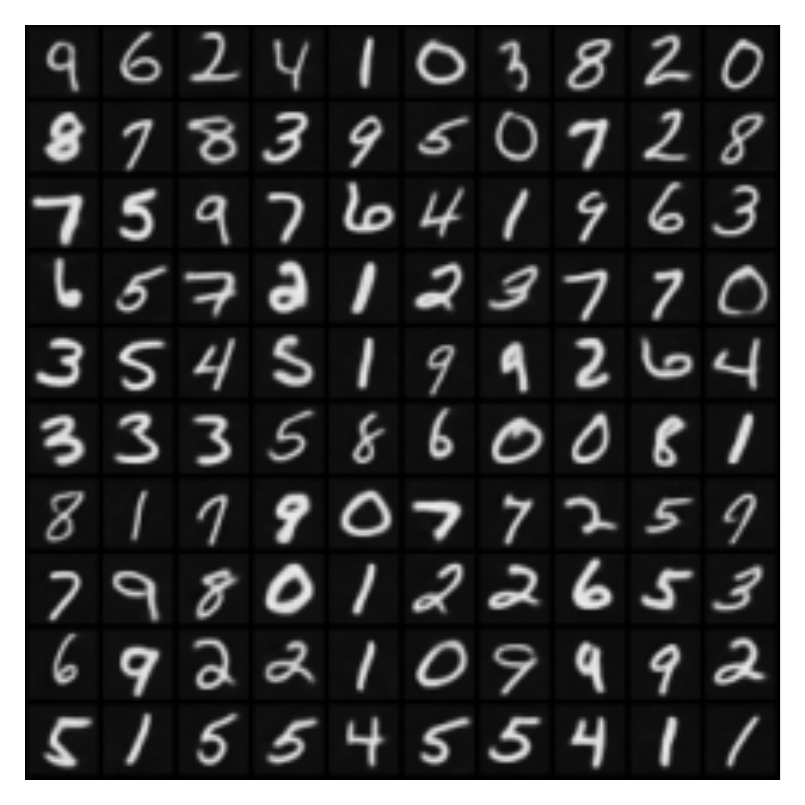}
     \caption{Reconstruction of batch.}
     \label{fig:Ng2}
  \end{subfigure}
  \begin{subfigure}[b]{0.3\textwidth}
     \includegraphics[width=1\linewidth]{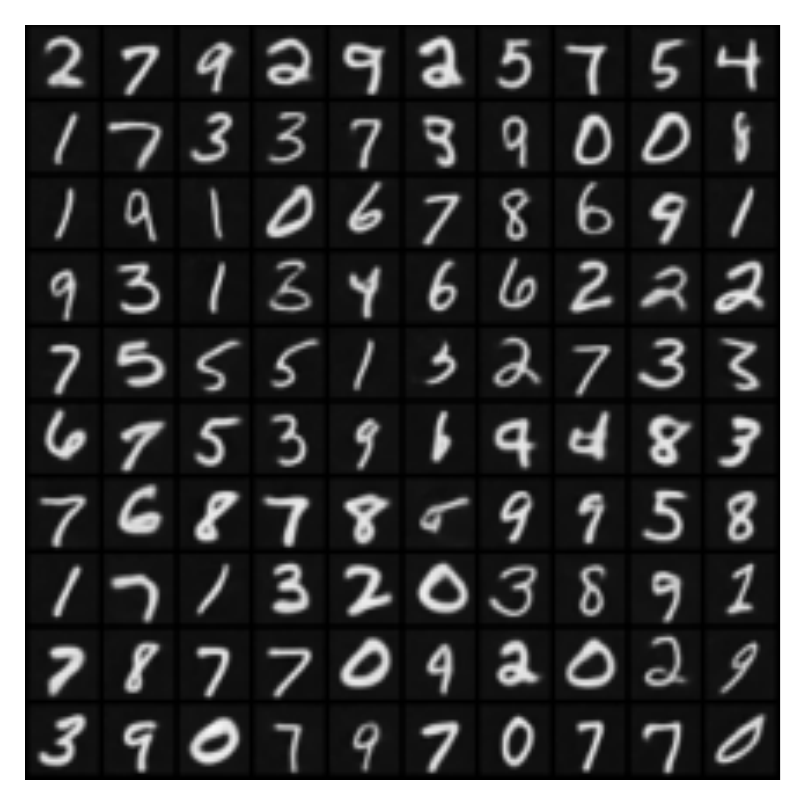}
     \caption{Random samples.}
     \label{fig:Ng2}
  \end{subfigure}
  \caption[]{Reconstruction and generation of images. We sampled a batch of $100$ images from the MNIST test set (left). We then generated trajectories where the selected actions were given by the GC-agent, in order to reconstruct the items in the batch (center). Finally, we randomly generated $100$ images with the S-agent (right). A qualitative analysis of the generated samples (right) shows that our generative model is able to produce a broad variety of images.}
  \label{fig:samples}
  \end{figure*}
  \vspace{-1.5cm}

We report our results in Table~\ref{tab:results}.
We first note that our method outperforms the VAE baseline in terms of MSE and FID by a large margin in both, test sample reconstruction and 
randomly generated images.
Our model lags behind DMs in terms of FID in both settings, reconstruction and randomly generated samples. 
This difference might be explained by the fact that we chose a very small horizon ($H=16$).
Exploring other values for $H$ is left as future work.

Our results confirm that the proposed model is not only able to accurately reconstruct the input but it is also able to randomly generate a rich variety 
of images.In Figure~\ref{fig:traj} we show two randomly generated trajectories, that is, trajectories produced by the S-agent, for the trained models 
in MNIST and Fashion-MNIST datasets.
A qualitative inspection of the reconstructed images and generated samples in Figure~\ref{fig:samples} shows
that our model effectively captures the underlying distribution of the data.
% \vspace{-.5cm}

%% file: main/conclusion.tex
%!TEX root = ../generative_by_rl.tex
\vspace{-.5cm}
\section{Conclusion and Discussion}
\label{sec:conclusion}

We introduced a novel framework to learn generative models based on goal-conditioned reinforcement learning.
The main idea of this work is to consider the elements of the training set as being generated by an agent that reaches those states after a fixed number of steps.
We then learn a mixture policy that approximates a family of goal-conditioned agents that are trained to generate trajectories that lead to the training points.
Following this line of reasoning we derived an upper-bound for the negative log-likelihood consisting of two terms.
A reconstruction error term that measures the reconstruction quality of the goal-conditioned policy, and a divergence term that encourages the mixture policy to remain close to the family of goal-conditioned agents.
Our experiments demonstrate that our method is able to effectively reconstruct the training set and generate a rich variety of outputs in the task of image synthesis.

There are some aspects of the algorithm we presented that are yet to be explored.
Regarding the use of GCRL to parameterize a generative model an interesting research direction could be to 
explore whether Hindsight Experience Replay \citep{andrychowicz2017} or a variant of it could contribute to 
obtaining a more effective GC-agent in terms of reconstruction loss.
We empirically showed that our method needs significantly less steps to reach the goals than DMs without compromising 
the sample quality (in terms of FID score).
A more thorough comparative study of the required number of steps by DMs and our method would be of interest.
Also, note that we have not use the state-of-the-art architectures for DMs nor for our algorithm due to time and resource 
constraints. 
Studying and comparing DMs and our method with more sophisticated models, and on bigger datasets, is left as future work. 

We believe that the framework we introduced could be useful for a broad scope of tasks. 
In particular, there are applications where the space of actions is inherently discrete.
For instance, in molecule design one could consider adding atoms and bonds to be the space of actions.
In this case one can dispense with a proposal function, and construct a model that only needs to learn to select the right 
actions.
Addressing this task is an interesting research direction.

%% file: appendix/detailed_derivation.tex
%!TEX root = ../generative_by_rl.tex

\section{Derivations of Section~\ref{sec:inference_as_GCRL}}
\label{app:proofs_derivations}

In this appendix we detail the missing derivations of Section~\ref{sec:inference_as_GCRL}.

\begin{lemma}
For any policy $\pi$, state $x$,
\label{lem:ub_neg_loglikelihood}
\[-\log\Big( \E_\pi \big[p_{H}(x|x_{H}, y_{H}) \big] \Big)
 \leq  \min_{\pi' \text{ policy}} \E_{\pi'} \left[\log\frac{1}{p_{H}(x|x_{H}, y_{H})}\right] +\KL(p^{\pi'},p^{\pi})\,.\]
\end{lemma}
\begin{proof}
Thanks to the Donsker-Varadhan's formula \citep{donsker1983} it holds
\[
-\log\Big( \E_\pi \big[p_{H}(x|x_{H}, y_{H}) \big] \Big) = \min_{q\in \Delta(\mathrm{Traj})} \E_q\left[\log\frac{1}{p_{H}(x|x_{H}, y_{H})} \right] + \KL(q,p^\pi)\,,
\]
where $\Delta(\mathrm{Traj})$ is the set of probability distributions supported on the trajectory $\tau = (x_1,y_1,\ldots,x_H,y_H)$ up to step $H$ with $x_h\in\cX$ and $y_h\in\cY$. Restricting the infimum to probability distributions $q=p^{\pi'}$ induced by some policy $\pi'$ allows us to conclude.
\end{proof}

%% file: appendix/experiments_detail.tex
%!TEX root = ../generative_by_rl.tex

\section{Network architectures}
\label{app:net_arch}
In this appendix we describe the architecture of the proposal, Q-value and selection networks.

\subsection{Proposal network}

The proposal network is parameterized as a four-layer \UNET\footnote{The \UNET implementation is publicly available at https://github.com/milesial/Pytorch-UNet.}.
Each layer has two blocks consisting of a convolution layer followed by group normalization and a ReLU activation function.

The different networks used in our model also depends on discrete time-step $h\in[H]$. We incorporate this information in the model as follows. At each time-step $h$ the proposal network receives an input of the form $(x_{h},h,a_h)$. The current state $x_h$ is of shape $w \times t \times c$, where $w$ is the width,
$t$ is the height, and $c$ is the number of channels
(for example, $c = 3$ for RGB images). The selected index $a_h$ indicates which of the $A$ available proposals is to be
chosen at the current time-step $h$.
We first calculate an embedding $\mathbf{h}_{a,h}$ representing $h$ and $a$ through an Embedding layer to
obtain a tensor the same size as $x_{h}$. We then concatenate $x_{h}$ and $\mathbf{h}_{a,h}$ on the channel axis to obtain a tensor of shape $w\times t\times 2c$ that is fed to the \UNET.
Finally, the \UNET outputs the tensor $x_{h+1}$ of shape $w\times t\times c$.
This process is described as follows
\begin{equation}
  \label{eq:propnet}
  \begin{aligned}
    h'_{a,h} &= Ah + a \\
    \mathbf{h}_{a,h} &= \mathtt{Emb}(h'_{a,h}) \\
    x_{h+1} &= \UNET([x_h, \mathbf{h}_{a,h}])\,,
  \end{aligned}
\end{equation}
where $[\cdot,\cdot]$ represents concatenation of 3D tensors on the channel axis. At each pass the network generates the proposal that corresponds to the specified index.
The output of each pass will then have the same shape as the inputs. We provide a detailed architecture of the network in Figure~\ref{fig:pnet}.

Note that, with this embedding of the time-step and selection the proposal network can memorize at most $H\times A$ distinct images. Thus combining proposals by choosing the selections in a trajectory is essential to be able to model a rich distribution.

\subsection{Selection network}

The selection network architecture is the same for both agents.
It consists of three blocks of convolution layers followed by a group normalization and a ReLU activation function.
The network outputs a tensor of size $\texttt{batch\_size}\times A$ corresponding to the selected actions for each
element in the batch. The embedding of the time-step follows the same method as in the proposal network.
Figure~\ref{fig:snet} presents a sketch of the architecture we used for our experiments.
The batch size was set to $128$, proposal size $A = 16$ and number of steps $H = 16$.

\subsection{Q-value selection network}

The architecture of the Q-value network is the same as the one of the selection network except for the last 
layer where the output is not normalized to obtain a probability distribution.
For the last layer we use a Dueling Network architecture \citep{wang2016dueling}.

\begin{figure}
  \includegraphics[width=\textwidth]{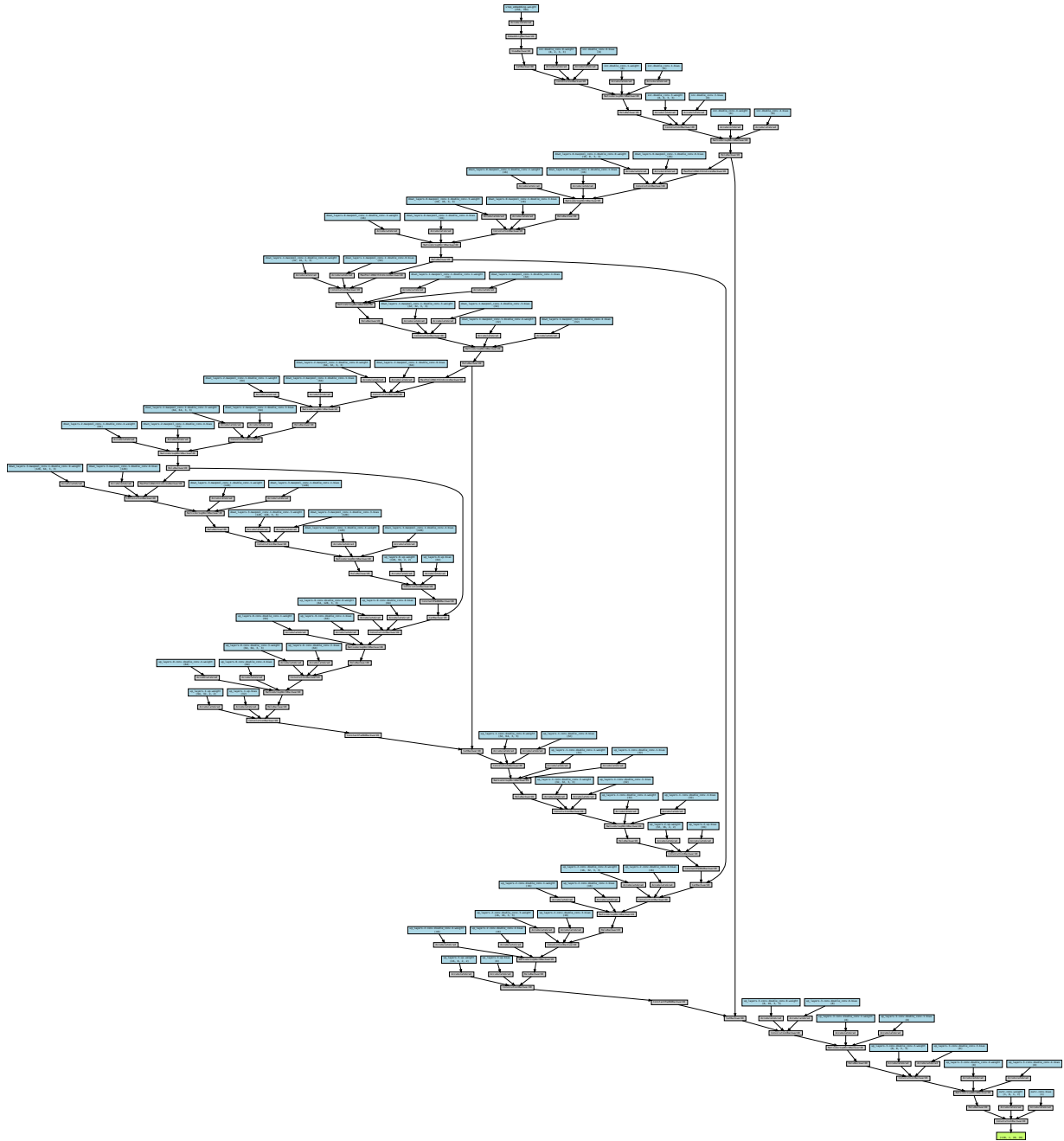}
  \caption{Proposal network architecture (shared between both agents) with batch size $128$ and number of steps $H = 16$.}
  \label{fig:pnet}
\end{figure}

\newpage
\begin{figure}
  \includegraphics[width=\textwidth]{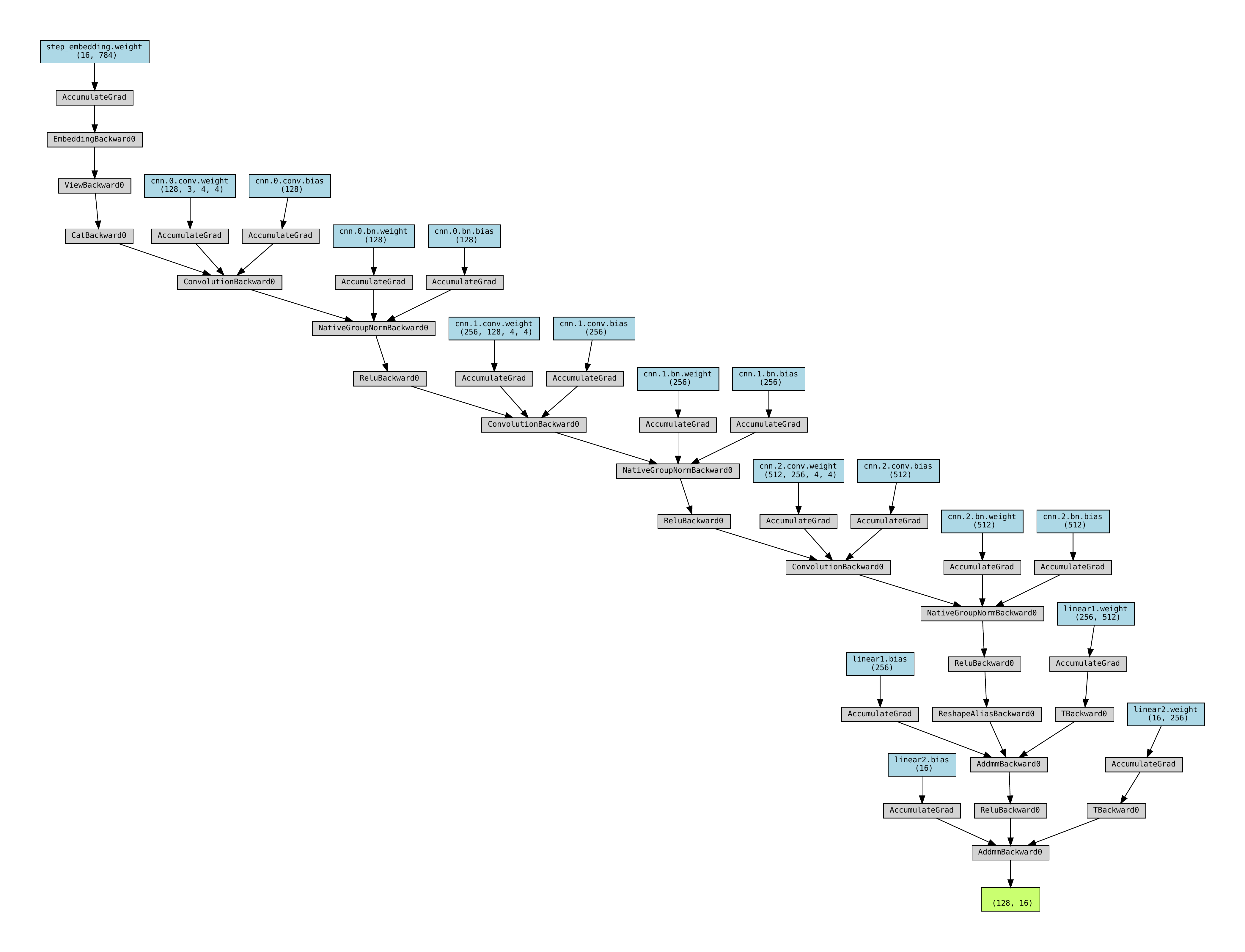}
  \caption{Selection network architecture for the GC-agent and S-agent with batch size $128$ and selection size $A = 16$.}
  \label{fig:snet}
\end{figure}